\documentclass{article}
\usepackage{fullpage,amsmath,amssymb, amsthm}
\usepackage[round]{natbib}

\newtheorem{thm}{Theorem}
\newtheorem{prop}[thm]{Proposition}

\newtheorem{lemma}[thm]{Lemma}

\newcommand{\R}{\mathbb{R}}

\newcommand{\sE}{{\mathcal E}}
\newcommand{\sF}{{\mathcal F}}

\newcommand{\sH}{{\mathcal H}}

\newcommand{\sX}{{\mathcal X}}

\renewcommand{\a}{\alpha}

\newcommand{\g}{\gamma}
\newcommand{\eps}{\epsilon}

\newcommand{\sign}{\mathop{\mathrm{sign}}}
\newcommand{\argmin}{\operatornamewithlimits{arg\ min}}

\newcommand{\ind}[1]{{\bf 1}_{\{#1\}}}

\newcommand{\abs}[1]{\left\lvert #1 \right\rvert}
\newcommand{\set}[1]{\left\{#1\right\}}

\newcommand{\brac}[1]{\left[#1\right]}

\newcommand{\bcase}{\left\{ \begin{array}{ll} }
\newcommand{\ecase}{\end{array} \right. }
\newcommand{\ev}{\mathbb{E}}
\newcommand{\ee}[2]{\ev_{#1}\brac{#2}}
\newcommand{\e}[1]{\mbe\brac{#1}}
\newcommand{\mbe}{\mathbb{E}}

\newcommand{\Ehat}{\widehat{\sE}}

\newcommand{\Yt}{\tilde{Y}}
\newcommand{\yt}{\tilde{y}}

\usepackage{multirow}
\usepackage{graphicx}

\begin{document}

%

%

\title{Learning from Multiple Corrupted Sources, \\ with Application to \\
Learning from Label Proportions}
\author{ Clayton Scott and Jianxin Zhang \\
University of Michigan \\ Electrical Engineering and Computer Science }
\maketitle


\begin{abstract}
We study binary classification in the setting where the learner is presented with multiple corrupted training samples, with possibly different sample sizes and degrees of corruption, and introduce an approach based on minimizing a weighted combination of corruption-corrected empirical risks. We establish a generalization error bound, and further show that the bound is optimized when the weights are certain interpretable and intuitive functions of the sample sizes and degrees of corruptions. We then apply this setting to the problem of learning with label proportions (LLP), and propose an algorithm that enjoys the most general statistical performance guarantees known for LLP. Experiments demonstrate the utility of our theory.
\end{abstract}

\section{INTRODUCTION}

We study the problem of binary classification in the setting where the learner does not have access to a conventional training data set with correctly labeled instances. Instead, the learner has access to several data sets for which the true labels have been randomly corrupted, with each data set having possibly different sample size and degree of corruption. Previous work has considered learning from a single corrupted data set, but the problem considered here raises the natural question of how best to aggregate and weight the information from these multiple corrupted data sets according to the sample size and degree of corruption. We extend the method of corruption corrected losses \citep{natarajan18jmlr} to this setting and establish a generalization error bound for kernel-based predictors. By optimizing this bound, we obtain a precise and interpretable scheme for aggregating the various corrupted sources according to the degree of corruption.

We then apply our framework to the problem of learning from label proportions (LLP), which is another weak supervision setting for binary classification. In this problem, training data come in the form of bags. Each bag contains unlabeled feature vectors (patterns) and is annotated with the proportion of patterns arising from class 1. We argue that this problem can be reduced to the first problem studied, and apply our results to obtain the most general theoretical analysis of this problem to date.


\subsection{Related Work and Contributions}

Our work is in the setting of classification with label-dependent label noise.
Prior work has focused on a single noisy dataset. \citet{blanchard16ejs} introduce a way to estimate the label noise proportions, while \citet{natarajan18jmlr} introduce a loss-correction framework when the noise proportions are known (or can be estimated). There have also been recent efforts to develop  noise-robust losses for training neural networks \citep{patrini17dnn,sabuncu18nips}.

\citet{rooyen18jmlr}, who extend the method of \citet{natarajan18jmlr} to other weakly supervised settings, considered the question of learning from multiple independent sources. Like us, they also prove a generalization error bound for the weighted empirical risk based on multiple corruption-corrected losses. However, their weights are fixed and depend only on the sample sizes, and they assume the corrected losses are bounded, which precludes many common losses. Relative to this work, our contribution is to prove a bound that holds for Lipschitz losses and arbitrary weights, with the optimized weights depending on both the sample sizes and degrees of corruption in an intuitive way.

The problem of learning a classifier from multiple corrupted datasets arises in some applications such as crowdsourcing and learning from multiple weak labeling functions \citep{ratner16nips}. In these applications, the corrupted datasets often arise from different annotators assigning labels to the same unlabeled datasets. This leads the datasets to be dependent, whereas our analysis assumes independence. However, if unlabeled data are very abundant as envisioned by \cite{ratner16nips}, independent unlabeled datasets could be provided to the different annotators, thus fitting our assumption.

LLP is motivated by a number of real-world applications including particle physics \citep{komiske18}, vote prediction \citep{sun17}, and image classification \citep{ding2017}. Several algorithms have been proposed in recent years, and we refer the reader to the recent summary in \citet{arnold19tr}.


Statistical analysis of LLP has been studied by a handful of papers. \citet{quadrianto09jmlr} reduce LLP to estimation of a certain mean operator, and use Rademacher complexity to establish uniform convergence guarantees for this estimate. \citet{nolabelnocry} extend the work of \citet{quadrianto09jmlr} in several ways, including a generalization error bound for LLP based on a notion of bag-Rademacher complexity, and present two algorithms for approximately optimizing the associated bag-empirical risk. Neither of these works establish consistency guarantees for a classification performance measure, and both focus on linear methods. \citet{yu15tr} study the algorithm of minimizing the ``empirical proportion risk," which seeks a classifier that best reproduces the bag label proportions. Bounds on the accuracy of the resulting classifier are established, but only under two restrictive assumptions, namely, that the bags are all very pure (label proportions close to 0 or 1) or that the feature space is finite. Relative to these works, our contributions are to establish distribution-free generalization error analysis and universal consistency with respect to a classification performance measure.



Our work may also be viewed as a novel setting for multi-source domain adaptation, where the target is the clean distribution. Typically, domain adaptation methods are classified as supervised, semi-supervised, or unsupervised, depending on whether the feature vectors from the target distribution are fully labeled, partially labeled, or unlabeled \citep{pan:10:kde}. In our case, however, we do not even assume access to unlabeled data from the target (clean) distribution.

Finally, we note that our framework yields a generalization error bound and consistency for the problem of classification with feature dependent label noise \citep{menon18,cannings18,scott19alt}, which has not previously been addressed. This follows by taking each corrupted data set to consist of a single point. We elaborate below.





\subsection{Notation and Terminology}

A binary classification loss function, referred to simply as a {\em loss} in this work, is a function $\ell:\R \times \{-1,1\} \to [0,\infty)$. We say a loss $\ell$ is $L$-Lipschitz if for every $y \in \{-1,1\}$, and every $t, t' \in \R$, $|\ell(t,y)-\ell(t',y)| \le L|t - t'|$. We say a loss is Lipschitz if it is $L$-Lipschitz for some $L$. We say a loss is a {\em margin loss} if it has the form $\ell(t,y) = \varphi(yt)$ for some $\varphi:\R \to [0,\infty)$.

A decision function is a measurable function $f:\sX \to \R$. The classifier induced by a decision function $f$ is the function $x \mapsto \sign(f(x))$. We will only consider classifiers induced by a decision function. In addition, we will often refer to a decision function as a classifier, in which case we mean the induced classifier.

Finally, define the probability simplex $\Delta^N := \{w \in \R^N \, | \, w_i \ge 0 \, \forall i, \text{ and } \sum_i w_i =1\}$.

\subsection{Outline}

Sec. \ref{sec:mult} presents our generalization error analysis for learning from multiple corrupted sources, and Sec. \ref{sec:llp} applies it to LLP. Discrimination rules are discussed in Sec. \ref{sec:discrim}, while Sec. \ref{sec:exp} present experiments illustrating our theory. Sec. \ref{sec:disc} offers some concluding comments.

\section{LEARNING FROM MULTIPLE CORRUPTED SOURCES}
\label{sec:mult}

We first review the work of \citet{natarajan18jmlr} before presenting two settings for learning from multiple corrupted sources.

\subsection{Learning from A Single Corrupted Source}

Let $\sX$ be the feature space and $P$ a probability distribution on $\sX \times \{-1,1\}$, which shall denote a clean/uncorrupted distribution. Let $\rho = (\rho^{+},\rho^{-}) \in [0,1]^2$ be such that $\rho^{-} + \rho^{+} < 1$, referred to as {\em label noise proportions}. Given $P$ and $\rho$, we define a new probability distribution $P^{\rho}$ on $\sX \times \{-1,1\}$ as follows. To generate a realization $(X,\Yt)$ from $P^{\rho}$, first draw a realization $(X,Y)$ from $P$, and then let the conditional distribution of $\Yt$ given $Y$ be
\begin{align*}
\Yt \, | \, Y = 1 \ & \sim \ \text{Ber}(1 - \rho^+) \\
\Yt \, | \, Y = -1 \ & \sim \ \text{Ber}(\rho^-).
\end{align*}
In words, the noisy label $\Yt$ is obtained by randomly corrupting the true label $Y$ in a manner depending on the value of $Y$, but not on $X$, such that $\rho^+$ is the probability that a $+1$ is flipped to a $-1$, and $\rho^-$ is the probability that a $-1$ is flipped to a $+1$. Below we use $p(\yt | y ; \rho)$ to denote this conditional distribution of $\Yt$ given $Y = y$, with noise proportions $\rho$.

Let $\ell:\R \times \{-1,1\} \to [0,\infty)$ be a loss. Following \citet{natarajan18jmlr}, given label noise proportions $\rho = (\rho^{+},\rho^{-})$, define the {\em corrected loss} $\ell^{\rho}$ by
\begin{align*}
\ell^{\rho}(t,1):= \frac{(1 - \rho^{-})\ell(t,1) - \rho^+ \ell(t,-1)}{1 - \rho^{-} - \rho^{+}} \\
\ell^{\rho}(t,-1):= \frac{(1 - \rho^{+})\ell(t,-1) - \rho^- \ell(t,1)}{1 - \rho^{-} - \rho^{+}}.
\end{align*}
The corrected loss (also referred to as an unbiased loss) debiases the affects of noisy labels as captured by the following result.

\begin{lemma}
\label{lem:corrected}
With $\ell, \rho$ and $\ell^\rho$ as in the preceding paragraph, for any $y \in \{-1,1\}$ and $t \in \R$,
$$
\mbe_{\Yt \sim p(\yt | y; \rho)}[\ell^\rho(t,\Yt)] = \ell(t,y).
$$
Additionally, consider any distribution $P$ on $\sX \times \{-1,1\}$ with associated corrupted distribution $P^\rho$. If $f$ is a decision function such that $\ell(f(\cdot),y)$ is bounded for all $y$, then
$$
\mbe_{(X,\Yt)\sim P^\rho}[\ell^\rho(f(X),\Yt)] = \ev_{(X,Y)\sim P}[\ell(f(X),Y)].
$$
\end{lemma}

The first part of the lemma was establish by \citet{natarajan18jmlr}, and the second part follows as an immediate corollary, where the boundedness assumption ensures the existence of the expectation. Given noisy data $(X_1,\Yt_1),\ldots,(X_n,\Yt_n) \sim P^\rho$, this result immediately suggests learning a classifier by empirical risk minimization over the noisy data,
$$
\widehat{f} = \argmin_{f \in \sF} \ \frac1{n} \sum_{j=1}^n \ell^\rho(f(X_j),\Yt_j),
$$
where $\sF$ is some class of decision functions. \citet{natarajan18jmlr} show that the generalization error, i.e., the difference between the  corrupted empirical risk and the uncorrupted true risk, is bounded uniformly over $\sF$, in terms of the Rademacher complexity of $\sF$.

Below, we extend the above framework to the setting where there are multiple contaminated sources with different degrees of corruption.

\subsection{Learning from Multiple Corruptions of a Common Source}
\label{sec:common}

We now consider the problem of learning from multiple corrupted sources. In this section, we assume observations are drawn from one of several distributions, all of which are corruptions of a common clean distribution $P$. Our goal is to learn a decision function $f$ to minimize the risk
\begin{equation}
\label{eqn:risk}
\sE_{P}(f) := \mbe_{(X,Y)\sim P} [\ell(f(X),Y)].
\end{equation}
The clean distribution is only observed through several corrupted samples. Thus, let $\rho_i = (\rho_i^{+},\rho_i^{-})$, $i=1,\ldots,N$ be $N$ pairs of label noise proportions, and let $P^{\rho_1}, \ldots, P^{\rho_N}$ be the associated corrupted distributions. For each $i$, let  $S_i = ((X_{ij}, \Yt_{ij}))_{1 \le j \le n_i}$ be $n_i$ iid draws from $P^{\rho_i}$, and assume the samples $S_1,\ldots,S_N$ are themselves independent.

To learn from this data, we propose to minimize a weighted sum of empirical risks
$$
\sum_{i=1}^N w_i \left[\frac1{n_i} \sum_{j=1}^{n_i} \ell^{\rho_i}(f(X_{ij}),\Yt_{ij}) \right]
$$
where $w \in \Delta^N$ is a tuning parameter.

Intuitively, samples $S_i$ with less corruption should receive larger weights. We confirm this intuition by establishing a generalization error bound, and then showing that the bound is indeed optimized by assigning larger weights to less corrupted samples.

To obtain an explicit bound, we specialize to kernel algorithms. Thus, let $k$ be a symmetric, positive definite (SPD) kernel, and let $B_k(R)$ denote the ball of radius $R>0$, centered at the origin, in the reproducing kernel Hilbert space (RKHS) associated to the kernel. Assume the kernel is bounded by the constant $K$. For example, in the case of the Gaussian kernel $k(x,x') = \exp(-\gamma \|x - x'\|^2)$, $K=1$.

For all generalization error bounds in this paper, we assume that $\ell$ is a {\em margin loss}, that is, $\ell(t,y) = \varphi(yt)$ for some $\varphi: \R \to [0,\infty)$. This assumption could be dropped at the expense messier expressions.
\begin{thm}
\label{thm:gen}
Let $\ell$ be an $L$-Lipschitz margin loss and $k$ an SPD kernel bounded by $K$. Let $P$ be a distribution on $\sX \times \{-1,1\}$, and for each $i$, let $\rho_i = (\rho_i^+, \rho_i^-) \in [0,1]^2$ such that $\rho_i^- + \rho_i^+ < 1$. For all $R > \varphi(0)/KL$, $w \in \Delta^N$, and $0 < \delta \le \frac14$, we have
with probability at least $1-\delta$ with respect to the draws of the corrupted samples $S_1, \ldots, S_N$,
\begin{equation}
\label{eq:gen}
\sup_{f \in B_{k}(R)} \abs{\sum_{i=1}^N w_i \left[ \frac1{n_i} \sum_{j=1}^{n} \ell^{\rho_i}(f(X_{ij}),\Yt_{ij}) \right] -  \sE_{P}(f)} \\
\leq 4 K R L \sqrt{\sum_{i=1}^N \frac{w_i^2}{n_i} \left( \frac{1+|\rho_i^+ - \rho_i^-|}{1-\rho_i^- - \rho_i^+} \right)^2  \left(\frac{\log (2/\delta)}{2} \right) }.
\end{equation}
\end{thm}

To interpret the bound, first consider the case where $n_i = n$ for all $i$, and $w_i = 1/N$. The bound becomes
$$
\frac{4 K R L}{\sqrt{Nn}}\sqrt{\frac1{N}\sum_{i=1}^N \left( \frac{1+|\rho_i^+ - \rho_i^-|}{1-\rho_i^- - \rho_i^+} \right)^2 \frac{\log (2/\delta)}{2}}
$$
which tends to zero provided {\em either} of $N$ or $n$ tends to $\infty$ (but not necessarily both), and $1-\rho_i^- - \rho_i^+$ remains bounded away from zero.

However, uniform weights do not optimize the bound. Returning to the setting where the $n_i$ are arbitrary, introduce the notation
\begin{equation}
\label{eqn:nsr}
c_i = \frac1{\sqrt{n_i}} \cdot \frac{1 + |\rho_i^+ - \rho_i^-|}{1 - \rho_i^- - \rho_+}.
\end{equation}
Define $1/c_i^2$ to be the ``signal-to-noise" ratio (SNR) for the $i$th sample, where $n_i$ reflects the ``signal"  while $(1 + |\rho_i^+ - \rho_i^-|)^2/(1 - \rho_i^- - \rho_+)^2$ captures the overall amount of noise. This notion of overall noise is smaller whenever $\rho_i^+$ and $\rho_i^-$ are small and when the label noise is {\em symmetric} \citep{kamalaruban18tr}.

It can easily be shown using Lagrange multiplier theory that the $w$ minimizing
\begin{equation}
\label{eqn:quadbnd}
\sum_{i=1}^N w_i^2  c_i^{2}
\end{equation}
subject to $w \in \Delta^N$ is $w_i = c_i^{-2}/\sum_j c_j^{-2}$, and that the optimal value of (\ref{eqn:quadbnd}) is
$\frac1{N}H(c_1^2, \ldots, c_N^2),$
where $H$ denotes the harmonic mean. The harmonic mean is dominated by the smallest $c_i$, with large ones having considerably less impact. In contrast, if we selected uniform weights $w_i = 1/N$, the harmonic mean is replaced by the arithmetic mean, $A(c_1^2, \ldots, c_N^2) = \frac1{N} \sum_i c_i^2$, which is much more sensitive to large values of $c_i$ (samples with very low SNR). As an illustration, suppose $N=10$, $n_i = n = 100$, and for $i < N$, $\rho_i^+ = \rho_i^- = 0.01$, while $\rho_N^+ = \rho_N^- = 0.49$. Then the ratio of arithmetic mean to harmonic mean exceeds 100.

In summary, the optimal weight $w_i$ is proportional to the signal-to-noise ratio of the $i$th sample. Thus, samples with larger sample size, less total label noise, and more balanced label noise, are preferred.

{\bf Additional remarks:}
{\bf (1)} Theorem \ref{thm:gen} and the other generalization error bounds below are all special cases of a master theorem provided in the supplemental material. The proof of the master theorem is based on an extension of Rademacher complexity analysis to a novel weighted Rademacher complexity. The factor $L\frac{1 + |\rho_i^+ - \rho_i^-|}{1 - \rho_i^- - \rho_+}$ in Thm. \ref{thm:gen} is the Lipschitz constant of $\ell^{\rho_i}$. In the event of a single sample with no noise, we recover known bounds \citep{mohri12} up to minor differences.

{\bf (2)} In the problem of binary classification with feature dependent label noise, there is a single dataset where each label is corrupted with probability $\rho^+(x)$ (if $Y=1$) or $\rho^-(x)$ (if $Y=-1$), where $x$ is the feature vector. This problem can be viewed as the special case where $n_i = 1$ and $\rho^{\pm}_i = \rho^{\pm}(x_i)$. Thus, Theorem \ref{thm:gen} applies to this setting provided the corruption probability functions $\rho^{\pm}(x)$ are known which, unfortunately, is rarely the case.

{\bf (3)} For the generalization error bound to imply a consistent learning procedure, it must converge to zero. We note one asymptotic setting where the bound does converge to zero for the optimized weights, but not for uniform weights. This is the setting where $N$ is fixed, and some but not all of the sample sizes $n_1, \ldots, n_N$ tend to $\infty$. To see that the optimized bound vanishes, just realize that the optimized weights lead to a smaller bound than when the $w_i$ are uniform on the samples with diverging sizes.

\subsection{Learning from Multiple Corrupted Sources with Varying Class Priors}
\label{sec:ber}

We now assume each corrupted source is associated to its own clean distribution, where the different clean distributions differ only in their class prior probabilities. In particular, let $P_1, \ldots, P_N$ denote the $N$ clean distributions, and let $\pi_i = P_i(Y=1)$. All $P_i$ are assumed to have the same class-conditional distributions, $P_+$ and $P_-$, which are two fixed distributions on $\sX$. This setting will be used to study LLP below.

In this setting it makes sense to look at a performance measure that depends only on the class-conditional distributions $P_+$ and $P_-$. Thus, we will adopt as our performance measure the balanced error rate (BER),
$$
\sE_{\pm}(f):= \frac12 \left( \underset{X\sim P_+}{\mbe}[\ell(f(X),1)] + \underset{X\sim P_-}{\mbe}[\ell(f(X),-1)] \right)
$$
The BER may be related to a cost-sensitive risk as follows. Let $\a = (\a^+,\a^-) \in \R_+^2$. Given a loss $\ell$, define the associated cost-sensitive loss
$$
\ell_\a(t,y):= \a_+ \ind{y=1} \ell(t,1) + \a_- \ind{y=-1} \ell(t,-1).
$$
Given a distribution $P$ on $\sX \times \{-1,1\}$, define the cost-sensitive risk with costs given by $\a$,
$$
\sE_{P,\a}(f) := \mbe_{(X,Y)\sim P} \ell_\a(f(X),Y).
$$
If $\pi$ is the prior class probability of $P$, and $\a = (\frac1{2\pi},\frac1{2(1-\pi)})$, then it can easily be shown that
$$
\sE_{P,\a}(f) = \sE_{\pm}(f).
$$

We seek to learn a classifier minimizing $\sE_{\pm}(f)$ given corrupted samples from each of $P_1, \ldots, P_N$. Thus, let $\rho_i = (\rho_i^{+},\rho_i^{-})$, $i=1,\ldots,N$ be $N$ pairs of label noise proportions, and let $P_1^{\rho_1}, \ldots, P_N^{\rho_N}$ be the associated corrupted distributions. For each $i$, let  $S_i = ((X_{ij}, \Yt_{ij}))_{1 \le j \le n_i}$ be $n_i$ iid draws from $P_i^{\rho_i}$, and assume  $S_1,\ldots,S_N$ are themselves independent.

We again propose to minimize a weighted sum of empirical risks
$$
\sum_{i=1}^N w_i \left[\frac1{n_i} \sum_{j=1}^{n_i} \ell_{\a_i}^{\rho_i}(f(X_{ij}),\Yt_{ij}) \right]
$$
where $\a_i = (\frac1{2\pi_i},\frac1{2(1-\pi_i)})$ and $\ell_{\a_i}^{\rho_i}$ denotes $(\ell_{\a_i})^{\rho_i}$, which is {\em not} the same as $(\ell^{\rho_i})_{\a_i}$. Adopting the same kernel setting as above, we have the following.
\begin{thm}
\label{thm:gen2}
Let $\ell$ be an $L$-Lipschitz margin loss, and let $k$ be an SPD kernel bounded by $K$. Let $P_+$ and $P_-$ be two probability distributions on $\sX$, and for $i=1,\ldots, N$, let $P_i$ be the distribution on $\sX \times \{-1,1\}$ having $P_+$ and $P_-$ as class-conditional distributions and $\pi_i$ as class prior. Set $\a_i = (\frac1{2\pi_i},\frac1{2(1-\pi_i)})$. For all $R > 2\varphi(0)/KL$, $w \in \Delta^N$, and $0 < \delta \le \frac14$, we have
with probability at least $1-\delta$ with respect to the draw of the corrupted samples $S_1, \ldots, S_N$,
\begin{equation*}
\sup_{f \in B_{k}(R)} \abs{
\sum_{i=1}^N w_i \left[ \frac1{n_i} \sum_{j=1}^{n} \ell_{\a_i}^{\rho_i}(f(X_{ij}),\Yt_{ij}) \right] -  \sE_{\pm}(f)} \\
\le C \sqrt{\sum_{i=1}^N \frac{w_i^2}{n_i} \left( \frac{1+|\rho_i^+ - \rho_i^-|}{1-\rho_i^- - \rho_i^+} \right)^2  \frac1{\min(\pi_i,1-\pi_i)^2}  },
\end{equation*}
where $C = 2 K R L \sqrt{\frac{\log (2/\delta)}{2}}$.
\end{thm}

Similar remarks apply to this result as were made in the previous section. The optimized weights are still proportional to an appropriately defined SNR, updated to include the class priors, where now there is a preference for more balanced samples, i.e., $\pi_i \approx \frac12$.

\section{Learning from Label Proportions}
\label{sec:llp}

In learning from label proportions, the learner is presented with several  pairs $(B_1,\g_1), \ldots, (B_M, \g_M)$. Each $B_i$ is a {\em bag} of $m_i$ unlabeled instances, and each $\g_i \in [0,1]$ is the proportion of instances from class 1 in the bag. The goal is to learn an accurate classifier as measured by some performance measure. Our approach is to frame LLP in the setting of Section \ref{sec:ber}, and so we adopt the BER as our performance measure. This choice is already a departure from prior work on LLP, which adopts a risk as the performance measure. The BER has the advantage of not being sensitive to changes in class prior at test time.

To reduce LLP to the setting of Section \ref{sec:ber}, we pair the bags, and view each pair of bags as a corrupted sample. Thus, we re-index the bags and let $(B_i^-,\g_i^-)$ and $(B_i^+,\g_i^+)$ constitute the $i$th pair of bags, and assume $\g_i^- < \g_i^+$. Also assume the total number of bags is $M=2N$, so that the number of pairs of bags is $N$.
We associate to every $X_{ij}$ in $B_i^-$ the noisy label $\Yt_{ij} = -1$, and to every $X_{ij}$ in $B_i^+$ the noisy label $\Yt_{ij}=1$. Finally, let $S_i$ denote the set of all pairs $(X_{ij},\Yt_{ij})$, $j=1,\ldots,n_i$ from the $i$th pair of bags.

\subsection{Probabilistic model for paired bags}

We propose a probabilistic model that connects LLP to the problem of learning from corrupted sources as described in Section \ref{sec:ber}. Let $P_+$ and $P_-$ be two probability distributions on $\sX$.

First, we define a joint distribution $Q_i$ on a realization $(X,\Yt)$ from the $i$th pair of bags as follows:  Assume $\Yt \sim \text{Ber}(\frac12)$, and that the conditional distribution of $X$ given $\Yt$ is
\begin{align*}
X &\sim \g_i^+ P_+ + (1-\g_i^+) P_- \qquad \text{given $\Yt = 1$}\\
X &\sim \g_i^- P_+ + (1-\g_i^-) P_- \qquad \text{given $\Yt = -1$}.
\end{align*}
The conditional distributions of $X | \Yt$ are consistent with $\g_i^+$ and $\g_i^-$ being the label proportions for the two bags, making this a reasonable model for LLP. The label proportions themselves are viewed as fixed and nonrandom\footnote{However, it is easy to extend our analysis to the setting where the bag proportions are also random, in which case the above model is conditioned on the $\gamma_i^{\pm}$. In this case, the joint distribution of bag proportions can be arbitrary as long as $\gamma_i^- < \gamma_i^+$ almost surely. The results are proved by first conditioning on, and the integrating out, the $\gamma_i^{\pm}$.}.

Next, we will construct a clean distribution $P_i$, and label noise proportions $\rho_i = (\rho_i^+, \rho_i^-)$, such that $Q_i = P_i^{\rho_i}$. These label {\em noise} proportions are not to be confused with the label proportions $\g_i^+, \g_i^-$ that annotate the bags, and in fact the former will be defined in terms of the latter. Let $P_i$ be the joint distribution of $(X,Y)$ whose class conditional distributions are $P_+$ and $P_-$, and whose prior class probability is
$$
\pi_i := \frac{\g_i^+ + \g_i^-}2.
$$
Also, define the label noise proportions $\rho_i = (\rho_i^+, \rho_i^-)$ according to
$$
\rho_i^+ := \frac{\g_i^-}{\g_i^- + \g_i+}, \qquad
\rho_i^- := \frac{1 - \g_i^+}{2 - \g_i^- - \g_i^+}.
$$
A simple calculation shows
\begin{equation}
\label{eqn:rhogam}
1-\rho_i^- - \rho_i^+ = \frac{\g_i^+ - \g_i^-}{(\g_i^+ + \g_i^-)(2 - \g_i^+ -\g_i^-)}
\end{equation}
which ensures that $\rho_i^- + \rho_i^+ < 1$, because $\g_i^+ - \g_i^- > 0$ by assumption.

We claim that $Q_i = P_i^{\rho_i}$. First,  consider the distribution of $\Yt$ under $P_i^{\rho_i}$. We have
\begin{equation*}
\Pr(\Yt = 1) = \pi_i (1-\rho_i^+) + (1-\pi_i) \rho_i^- \\
= \frac{\g_i^- + \g_i^+}2 \cdot \frac{\g_i^+}{\g_i^- + \g_i^+} + \frac{2 -\g_i^- - \g_i^+}2 \cdot \frac{1 - \g_i^+}{2 - \g_i^- - \g_i^+},
\end{equation*}
which equals $\frac12$ in agreement with $Q_i$. Next, consider the conditional distribution of $X$ given $\Yt$ under both $Q_i$ and $P_i^{\rho_i}$. To show that these agree, it suffices to show $\Pr(Y=1 \, | \, \Yt=1) = \g_i^+$ and $\Pr(Y=1 \, | \, \Yt=-1) = \g_i^-$. This follows from
\begin{equation*}
\Pr(Y=1 \, | \, \Yt = 1) = \frac{\Pr(\Yt = 1 \, | \, Y=1) \Pr(Y=1)}{\Pr(\Yt = 1)} \\
= \frac{(1-\rho_i^+)\frac{\g_i^- + \g_i^+}{2}}{\frac12}
= \frac{\g_i^+}{\g_i^- + \g_i^+} (\g_i^- + \g_i^+)
= \g_i^+,
\end{equation*}
and the other case is similar.

\subsection{A generalization error bound for LLP}

From the preceding discussion, we may view each pair of bags as a corrupted sample, and since each $P_i$ has the same class-conditional distributions, we are in the setting of Section \ref{sec:ber} and may deduce the following result. The result is actually obtained by applying the master theorem in the appendix, which allows us to save a factor of 3 in the leading constant compared with applying Theorem \ref{thm:gen2} directly.
\begin{thm}
\label{thm:gen4}
Let $\ell$ be an $L$-Lipschitz margin loss, and let $k$ be a SPD kernel bounded by $K$. Let $P_+$ and $P_-$ be two probability distributions on $\sX$. Let $S_1, \ldots, S_N$ be pairs of bags described by the construction and probabilistic model above. For each $i$, let (a) $\g_i^- < \g_i^+$ be the label proportions for the $i$th pair of bags, (b) $\pi_i = \frac{\g_i^+ + \g_i^-}2$, (c) $P_i$ be the distribution on $\sX \times \{-1,1\}$ with $P_+$ and $P_-$ as class-conditional distributions and $\pi_i$ as the class prior (d) $\a_i = (\frac1{\g_i^+ + \g_i^-},\frac1{2-\g_i^+ - \g_i^-})$, and (e) $\rho_i = (\frac{\g_i^-}{\g_i^- + \g_i^+}, \frac{1 - \g_i^+}{2 - \g_i^- - \g_i^+})$. For all $R > \varphi(0)/KL$, $w \in \Delta^N$, and $0 < \delta \le \frac14$, we have
with probability at least $1-\delta$ with respect to the draw of the corrupted samples $S_1, \ldots, S_N$,
\begin{equation*}
\sup_{f \in B_{k}(R)} \abs{
\sum_{i=1}^N w_i \left[ \frac1{n_i} \sum_{j=1}^{n_i} \ell_{\a_i}^{\rho_i}(f(X_{ij}),\Yt_{ij}) \right] -  \sE_{\pm}(f)} \\
\le 4 K R L \sqrt{\sum_{i=1}^N \frac{w_i^2}{n_i (\g_i^+ - \g_i^-)^2} \left(\frac{\log (2/\delta)}{2} \right)} .
\end{equation*}
\end{thm}

\subsection{Optimal pairing of bags}
\label{sec:pair}

The bound above is minimized by selecting weights
$$
w_i \propto n_i (\g_i^+ - \g_i^-)^2,
$$
which gives preference to pairs of bags where one bag is mostly +1's (large $\g_i^+$) and the other is mostly -1's (small $\g_i^-$). With these weights, the bound becomes
$$
4 K R L \sqrt{\left( \sum_{i=1}^N n_i (\g_i^+ - \g_i^-)^2 \right)^{-1} \frac{\log (2/\delta)}{2} }.
$$

We can also optimize the pairing of bags by further optimizing the bound.
Consider the unpaired bags $(B_i, \g_i)$, $i=1,\ldots, 2N$, and assume all unpaired bags have the same size\footnote{We make this assumption because for our probabilistic model, bags in a pair have the same (expected) size.}. We would like to pair each bag to a different bag, forming pairs $(\g_i^+, \g_i^-)$, such that
$$
\sum_{i=1}^N (\g_i^+ - \g_i^-)^2
$$
is maximized. For each $i < j$, let $u_{ij}$ be a binary variable, with $u_{ij}=1$ indicating that the $i$th and $j$th bags are paired. The optimal pairing of bags is given by the solution to the following integer program:
\begin{align*}
\max_{u}  & \ \ \sum_{1 \le i < 2N} \sum_{i < j \le 2N} (\g_i - \g_j)^2 u_{ij} \\
 \text{s.t.} & \ \ u_{ij} \in \{0,1\}, \forall i,j \\
& \ \  \sum_{i < j} u_{ij} + \sum_{j < i} u_{ji} = 1, \forall i
\end{align*}
The equality constraint ensures that every bag is paired with precisely one other distinct bag. This problem is known as the ``maximum weighted (perfect) matching" problem, and in our experiments we solve it using the algorithm of \citet{edmonds1965maximum}.


\section{Discrimination Rules}
\label{sec:discrim}

The bounds above all show that a weighted empirical risk of the form
$$
\Ehat(f) = \sum_i \frac{w_i}{n_i} \sum_j (\ell_{\a_i})^{\rho_i} (f(X_{ij}),\Yt_{ij})
$$
is, with high probability, uniformly close to an error quantity $\sE(f)$ ($=\sE_P(f)$ or $\sE_{\pm}(f)$ depending on context) over a ball in a RKHS. This suggests a discrimination rule based on regularized, weighted empirical risk minimization (ERM):
\begin{equation}
\label{eqn:discrule}
\widehat{f}:= \argmin_{f \in \sH_k} \ \Ehat(f) + \lambda \| f \|_{\sH_k}^2,
\end{equation}
where $\lambda \ge 0$ and $\sH_k$ is the RKHS associated to $k$. In this section we discuss statistical and computational aspects of this learning rule.

\subsection{Consistency}

An algorithm is consistent if the sequence of random variables $\sE(\widehat{f})$ (where the randomness arises from the training data) converges asymptotically to $\sE^* := \sE(f^*)$, where $f^*$ is a global minimizer of $\sE(f)$.
To prove consistency using our generalization error bounds,
 two things must happen asymptotically: (1) 
$B_k(R)$ can approximate $f^*$ to arbitrary accuracy as $R \to \infty$, and (2) the generalization error bound tends to zero. (1) is ensured by taking $k$ to be a universal kernel on a compact domain \citep{steinwart08}. (2) Requires balancing the sample sizes with a shrinking value of $\lambda$, which implicitly determines $R$.
These ideas are captured in the following general result that covers all three settings discussed above. We say that the problems in Secs. \ref{sec:common}, \ref{sec:ber}, and \ref{sec:llp} satisfy the {\em bounded noise condition} if $1-\rho_i^+ - \rho_i^-, (1-\rho_i^+ - \rho_i^-)\min(\pi_i,1-\pi_i)$, and $\g_i^+ - \g_i^-$ are bounded away from zero, respectively.
\begin{thm}
\label{thm:llpconsist}
Let $\sX$ be compact and let $k$ be a universal kernel on $\sX$. Let $\lambda$ be such that $\lambda \to 0$ and $\lambda (n_1 + \cdots + n_N)/\log(n_1 + \cdots + n_N) \to \infty$ as $n_1 + \cdots + n_N \to \infty$. For each of the settings in Secs. \ref{sec:common}, \ref{sec:ber}, \ref{sec:llp}, under the bounded noise condition, and with optimized weights,
\begin{equation}
    \sE(\widehat{f}) \to \sE^*
\end{equation}
in probability as $n_1 + \cdots + n_N \to \infty$.
\end{thm}
It is important to note that the condition $n_1 + \cdots + n_N \to \infty$ covers a wide range of asymptotic regimes. It essentially requires that at least one of $n_1, \ldots, n_N$ or $N$ tend to infinity. In particular, consistency is possible even though some sample sizes may remain bounded.

This consistency result is universal in the sense that it makes no assumptions on $P_-, P_+$ other than compact support. In the case of LLP, the only prior work to consider consistency makes much stronger distributional assumptions \citep{yu15tr}.

\subsection{Calibrated Losses}

The consistency discussion above applies to the loss function that is used in the algorithm. This loss is typically chosen for computational properties such as convexity, whereas the true loss of interest may be the 0-1 loss. Fortunately, we can ensure consistency wrt the 0-1 loss simply by choosing $\ell$ to be classification calibrated which, in the case of a margin loss $\varphi$, requires that it be differentiable at 0 and $\varphi'(0) < 0$ \citep{bartlett06}. This decoupling of calibration and noise correction is one nice feature of the method of \citet{natarajan18jmlr}.


\subsection{Convexity}

We say that the loss $\ell$ is {\em convex} if, for each $y$, $\ell(t,y)$ is a convex function of $t$. \citet{natarajan18jmlr} establish the following result, where $\ell''$ denotes the second derivative of $\ell$ with respect to its first variable.
\begin{prop}
\label{prop:conv1}
Let $\ell$ be a convex, twice differentiable loss satisfying
\begin{equation}
\label{eqn:second}
\ell''(t,1) = \ell''(t,-1)
\end{equation}
for all $t \in \R$. Then for any $\rho = (\rho^+,\rho^-) \in [0,1]^2$ satisfying $\rho^+ + \rho^- < 1$, $\ell^\rho$ is convex.
\end{prop}
This result follows easily from the definition of the corrected loss. Examples of losses satisfying the second order condition include the logistic, Huber, and squared error losses.

This result immediately implies that the weighted average of losses in Section \ref{sec:common}, $\sum_i \frac{w_i}{n_i} \sum_j \ell^{\rho_i}(t,\Yt_{ij})$ is a convex function of $t$, provided $\ell$ satisfies \eqref{eqn:second}. The result is less helpful in the case of Section \ref{sec:ber}, because cost-sensitive losses $\ell_\a$ do not generally satisfy \eqref{eqn:second}. In the setting of LLP, however, the choices of $\rho_i$ and $\a_i$ play together nicely and allow us to establish the following.

\begin{prop}
\label{prop:conv2}
Let $\ell$ be a convex, twice differentiable loss satisfying
\begin{equation}
\label{eqn:second2}
\ell''(t,1) = \ell''(t,-1)
\end{equation}
for all $t \in \R$. For $1 \le i \le N$, let $\g_i^+, \g_i^-, \a_i^+, \a_i^-, \rho_i^+$ and $\rho_i^-$ be as in the statement of Theorem \ref{thm:gen4}, and let $n_i^{\pm} := |\{j : \Yt_{ij} = \pm 1\}|$. Then
$$
J(t):=\sum_i \frac{w_i}{n_i} \sum_j (\ell_{\a_i})^{\rho_i} (t,\Yt_{ij})
$$
is a convex function of $t$ provided
\begin{equation}
\label{eqn:convsuff}
\sum_i \frac{w_i}{n_i(\g_i^+ - \g_i^-)} \Bigg[ n_i^+ \Bigg(\frac12 - \g_i^-\Bigg) + n_i^- \Bigg(\g_i^+ - \frac12\Bigg) \Bigg] \ge 0.
\end{equation}
If, in addition, $\ell''(t,y) > 0$ for all $t,y$, then \eqref{eqn:convsuff} is also necessary for convexity of $J(t)$.
\end{prop}

The condition in \eqref{eqn:convsuff} is satisfied if either (a) for all $i$, $\g_i^- \le \frac12$ and $\g_i^+ \ge \frac12$, or (b) for all $i$, $n_i^+ = n_i^-$. The condition in \eqref{eqn:convsuff} may still be satisfied even if neither (a) nor (b) are, and can be easily verified numerically for any instance of the problem.

\subsection{LLP Implementation}

We employ the logistic loss and use a Gaussian kernel.
To compute the minimizer of \eqref{eqn:discrule}, we appeal to the representer theorem \citep{scholkopf01} which implies that the solution has the form
$    f(x) = \sum_{i} \alpha_i k(x, x_i)$.
 In our LLP experiments, all (unpaired and therefore paired) bags have equal size. Prop. \ref{prop:conv2} thus implies that the objective is convex in the $\alpha_i$, which are learned via gradient descent. We also perform the optimal pairing of bags as described in Sec. \ref{sec:pair}, and use the weights $w_i \propto (\g_i^+ - \g_i^-)^2$ {\em as prescribed by our theory}. It is possible that parametrizing the weights in terms of $(\g_i^+ - \g_i^-)$, and tuning them by cross-validation, may lead to improved performance. If a bag pair has a weight of zero, i.e., $\gamma_i^+ = \gamma_i^-$, we drop the pair from the training data. Finally, we remark that the $\g_i$'s are fixed population level parameters in our model, whereas they are typically empirical quantities (computed from true but unobserved class labels) in practice. We take the latter to be estimates of the former.



\section{LLP EXPERIMENTS}
\label{sec:exp}

To illustrate our theory, we compare our approach against InvCal \citep{rueping10} and alternate-$\propto$SVM \citep{yu13}, which are the two most common reference methods in learning with label proportions.

\subsection{Datasets}
We employ datasets from the UCI repository (see table). For datasets with multiple classes, we convert them to binary by choosing one class as positive and the others are merged to form the negative class.

\begin{center}
\resizebox{0.5\textwidth}{!}{%
 \begin{tabular}{c || c || c || c}
 \hline
 Dataset & Size & Classes & Prior label distribution \\ [0.5ex]
 \hline\hline
 Banknote & 1372 & 2 & (55.54\%, 44.46\%) \\
 \hline
 Cardiotocography & 2126 & 3 & (22.15\%, 77.85\%) \\
 \hline
 CNAE & 1080 & 9 & (11.11\%, 88.89\%) \\
 \hline
 Contraceptive & 1473 & 3 & (42.70\%, 57.30\%) \\
 \hline
 German bank & 1000 & 2 & (30.00\%, 70.00\%) \\
 \hline
 Sports article & 1000 & 2 & (63.50\%, 36.50\%) \\
 \hline
 Wireless Indoor & 2000 & 4 & (25.00\%, 75.00\%) \\
 \hline
\end{tabular}
}
\end{center}

\subsection{Experiment Setup}
We first randomly split each dataset into training and test sets, using 80\% for training. The training set is then split into bags of fixed sizes. We consider the bag sizes of 2, 4, 8, 16, 32, 64, and select the parameters of each model by 5-fold cross validation. As we assume instance-level labels are inaccessible during training, the parameters of our method are selected by the corrected empirical risk, while the other two methods use the empirical proportion risk. The above process is repeated for 5 times per dataset per bag size. The mean accuracy and the mean balanced accuracy are reported with their standard deviations.

For InvCal, the parameters are tuned from $C_p \in \{0.1, 1, 10 \}$, and $\epsilon \in \{0, 0.01, 0.1\} $. For alter-$\propto$SVM, the parameters are tuned from $C \in \{0.1, 1, 10\}, C_p \in \{1, 10, 100\}.$ Following \citep{yu13}, The alternate $\propto$SVM is randomly initialized 10 times and the result with the smallest objective values is picked. In our approach, the parameters are tuned from learning rate $\in \{ 0.1, 0.01, 0.001 \}$, decay $\in \{0.01, 0.001, 0.0001 \}$, iterations $\in \{100, 800, 1600, 3200\}$, and regularization parameter $\in \{0.001, 0\}$.
For all approaches, the Gaussian kernel is used and the parameter $\gamma$ is tuned from \{0.001, 0.1, 1\}

\subsection{Results}
Performance is measured by balanced accuracy (1 - BER) with 0/1 loss, which is shown in Table 1 in the appendix. As a reference, the conventional accuracy is shown in Table 2. As shown by Table 1, corrected loss outperforms the other two methods 16 times, while for InvCal the number is 13 and for $\propto$SVM 15. We can conclude that our method is comparable to the other two methods for balanced accuracy.

\section{DISCUSSION}
\label{sec:disc}

We have presented one approach to learning from multiple corrupted sources, building on the method of unbiased losses  \citep{natarajan18jmlr}. At least two other approaches may be worth exploring. First, \citet{natarajan18jmlr} present a second general approach, the method of calibrated surrogates. With this method, it is no longer possible to decouple  calibration and noise correction, and different techniques would be needed. Second, one could use {\em symmetric} losses, which are immune to class-conditional label noise \citep{rooyen15tr,charoenphakdee19icml}.

We also point out that while our approach has focused on kernel methods for binary classification, it can be generalized. For example, unbiased losses for multiclass have be developed in recent years \citep{patrini17dnn, rooyen18jmlr}. Furthermore, our analysis is based on Rademacher complexity and should be extensible to other types of classifiers \citep{neyshabur15colt}.




\bibliographystyle{plainnat}

\bibliography{myBib}

\onecolumn

\section{DETAILS OF EXPERIMENTAL RESULTS}

\begin{table*}[h] \caption {Balanced accuracy}
\resizebox{1.0\textwidth}{!}{%
  \begin{tabular}{c||c||c|c|c|c|c|c}
    \hline
    Dataset & Method & 2 & 4 & 8 & 16 & 32 & 64 \\
    \hline
    \multirow{3}{*}{Banknote} & InvCal &
    \textbf{0.9994 $\pm$ 0.0013} &
    \textbf{0.9994 $\pm$ 0.0015} &
    0.9994 $\pm$ 0.0014 &
    0.9927 $\pm$ 0.0061 &
    \textbf{0.9636 $\pm$ 0.0089} &
    \textbf{0.9053 $\pm$ 0.0369} \\
    & alter-pSVM &
    \textbf{0.9994 $\pm$ 0.0013} &
    \textbf{0.9994 $\pm$ 0.0015} &
    \textbf{1.0000 $\pm$ 0.0000} &
    \textbf{0.9987 $\pm$ 0.0018} &
    0.8969 $\pm$ 0.0771 &
    0.7244 $\pm$ 0.1487 \\
    & corrected loss &
    0.9895 $\pm$ 0.0085 &
    0.9668 $\pm$ 0.0260 &
    0.9431 $\pm$ 0.0277 &
    0.8995 $\pm$ 0.0415 &
    0.7976 $\pm$ 0.1065 &
    0.7400 $\pm$ 0.1614 \\
    \hline
    \multirow{3}{*}{Cardiotocography} & InvCal &
    0.8286 $\pm$ 0.0079 &
    \textbf{0.8349 $\pm$ 0.0274} &
    0.7078 $\pm$ 0.1027 &
    \textbf{0.5805 $\pm$ 0.0912} &
    0.5340 $\pm$ 0.1169 &
    \textbf{0.4978 $\pm$ 0.0048} \\
    & alter-pSVM &
    \textbf{0.8465 $\pm$ 0.0083} &
    0.8241 $\pm$ 0.0555 &
    0.6814 $\pm$ 0.1749 &
    0.4424 $\pm$ 0.0060 &
    0.4374 $\pm$ 0.0047 &
    0.4471 $\pm$ 0.0081 \\
    & corrected loss &
    0.7985 $\pm$ 0.0779 &
    0.7616 $\pm$ 0.1278 &
    \textbf{0.7137 $\pm$ 0.0747} &
    0.5273 $\pm$ 0.1234 &
    \textbf{0.5407 $\pm$ 0.1576} &
    0.4925 $\pm$ 0.1457 \\
    \hline
    \multirow{3}{*}{CNAE} & InvCal &
    0.7727 $\pm$ 0.0504 &
    0.6974 $\pm$ 0.0694 &
    0.5372 $\pm$ 0.0231 &
    0.5000 $\pm$ 0.0000 &
    0.5000 $\pm$ 0.0000 &
    0.5000 $\pm$ 0.0000\\
    & alter-pSVM &
    0.9207 $\pm$ 0.0112 &
    0.7418 $\pm$ 0.1769 &
    0.7101 $\pm$ 0.1576 &
    0.5635 $\pm$ 0.1694 &
    0.4887 $\pm$ 0.0305 &
    0.4726 $\pm$ 0.0282 \\
    & corrected loss &
    \textbf{0.9367 $\pm$ 0.0174} &
    \textbf{0.8049 $\pm$ 0.0537} &
    \textbf{0.7552 $\pm$ 0.0634} &
    \textbf{0.6047 $\pm$ 0.0901} &
    \textbf{0.6055 $\pm$ 0.1270} &
    \textbf{0.5363 $\pm$ 0.0473} \\
    \hline
    \multirow{3}{*}{Contraceptive} & InvCal &
     0.6593 $\pm$ 0.0103 &
     0.6569 $\pm$ 0.0093 &
     \textbf{0.6096 $\pm$ 0.0428} &
     \textbf{0.5835 $\pm$ 0.0409} &
     \textbf{0.5908 $\pm$ 0.0169} &
     \textbf{0.5162 $\pm$ 0.0194} \\
    & alter-pSVM &
     0.6600 $\pm$ 0.0127 &
    0.6041 $\pm$ 0.0237 &
    0.5628 $\pm$ 0.0205 &
    0.5321 $\pm$ 0.0475 &
    0.5491 $\pm$ 0.0325 &
    0.4917 $\pm$ 0.0650 \\
    & corrected loss &
    \textbf{ 0.6616  $\pm$  0.0267} &
    \textbf{0.6992  $\pm$  0.0071} &
    0.5698  $\pm$  0.0834 &
    0.5172  $\pm$ 0.0501 &
    0.5525  $\pm$  0.0328 &
    0.5035  $\pm$  0.0191 \\
    \hline
    \multirow{3}{*}{German bank} & InvCal &
      0.6450 $\pm$ 0.0272 &
 0.5420 $\pm$ 0.0426 &
 0.5233 $\pm$ 0.0189 &
 \textbf{0.5371 $\pm$ 0.0321} &
 0.5094 $\pm$ 0.0209 &
 0.5000 $\pm$ 0.0000\\
    & alter-pSVM &
      0.6600 $\pm$ 0.0127 &
 0.6041 $\pm$ 0.0237 &
 0.5628 $\pm$ 0.0205 &
 0.5321 $\pm$ 0.0475 &
\textbf{ 0.5491 $\pm$ 0.0325} &
 0.4917 $\pm$ 0.0650\\
    & corrected loss &
     \textbf{0.6885 $\pm$ 0.0380} &
  \textbf{0.6572  $\pm$  0.0463} &
  \textbf{0.6285  $\pm$  0.0447} &
  0.5123  $\pm$  0.0647 &
  0.4869  $\pm$  0.0730 &
  \textbf{0.5353  $\pm$  0.0593}\\
  \hline
    \multirow{3}{*}{Sports article} & InvCal &
    0.8193 $\pm$ 0.0182 &
 0.7999 $\pm$ 0.0481 &
 0.8030 $\pm$ 0.0199 &
 0.7896 $\pm$ 0.0254 &
 0.7447 $\pm$ 0.0371 &
 0.5993 $\pm$ 0.1149\\
    & alter-pSVM &
       0.8107 $\pm$ 0.0059 &
 0.8266 $\pm$ 0.0445 &
 \textbf{0.8086 $\pm$ 0.0288} &
 \textbf{0.7993 $\pm$ 0.0103} &
 \textbf{0.7683 $\pm$ 0.0332} &
 \textbf{0.7059 $\pm$ 0.2231}\\
    & corrected loss &
       \textbf{0.8216  $\pm$  0.0201} &
  \textbf{0.8291  $\pm$  0.0341} &
  0.8050  $\pm$  0.0179 &
  0.6828  $\pm$  0.1895 &
  0.7610  $\pm$  0.0394 &
 0.6360 $\pm$  0.1387\\
 \hline
    \multirow{3}{*}{Wireless Indoor} & InvCal &
        0.9937   $\pm$  0.0042 &
 0.9898   $\pm$  0.0035 &
 0.9893   $\pm$  0.0082 &
 0.9882   $\pm$  0.0078 &
 \textbf{0.9679   $\pm$  0.0142} &
 0.6842   $\pm$  0.1041\\
    & alter-pSVM &
        \textbf{0.9962   $\pm$  0.0045} &
 \textbf{0.9983   $\pm$  0.0020} &
 \textbf{0.9947   $\pm$  0.0058} &
 \textbf{0.9967   $\pm$  0.0024} &
 0.8735   $\pm$  0.2680 &
 \textbf{0.9928   $\pm$  0.0053}\\
    & corrected loss &
         0.9724    $\pm$  0.0080 &
  0.9693    $\pm$   0.0234 &
  0.9488    $\pm$  0.0121 &
 0.9087   $\pm$  0.0342 &
  0.8646    $\pm$   0.0732 &
  0.7648    $\pm$   0.2458\\
  \end{tabular}
}
\end{table*}

\begin{table*}[h] \caption {Accuracy}
\resizebox{1.0\textwidth}{!}{%
  \begin{tabular}{c||c||c|c|c|c|c|c}
    \hline
    Dataset & Method & 2 & 4 & 8 & 16 & 32 & 64 \\
    \hline
    \multirow{3}{*}{Banknote}
    & InvCal & 0.9993 $\pm$ 0.0016 & 0.9993 $\pm$ 0.0016 & 0.9993 $\pm$ 0.0016 & 0.9923 $\pm$ 0.0063 & 0.9648 $\pm$ 0.0086 & 0.9115 $\pm$ 0.0365 \\
    & alter-pSVM &
     0.9993 $\pm$ 0.0016 &
 0.9993 $\pm$ 0.0016 &
 1.0000 $\pm$ 0.0000 &
 0.9986 $\pm$ 0.0019 &
 0.9007 $\pm$ 0.0738 &
 0.7264 $\pm$ 0.1477\\
    & corrected loss &
     0.9891 $\pm$ 0.0089 &
  0.9674  $\pm$ 0.0258 &
  0.9423  $\pm$  0.0274 &
  0.9007  $\pm$  0.0420 &
  0.7951  $\pm$  0.1065 &
  0.7385  $\pm$  0.1614\\
    \hline
    \multirow{3}{*}{Cardiotocography}
    & InvCal &
     0.8540 $\pm$ 0.0051 &
      0.8335 $\pm$ 0.0151 &
      0.8130 $\pm$ 0.0444 &
      0.7502 $\pm$ 0.0372 &
      0.7571 $\pm$ 0.0629 &
      0.7961 $\pm$ 0.0106
    \\
    & alter-pSVM &
      0.8146 $\pm$ 0.0136 &
   0.8037 $\pm$ 0.0311 &
  0.7609 $\pm$ 0.0759 &
  0.6730 $\pm$ 0.0097 &
  0.6719 $\pm$ 0.0100 &
  0.6900 $\pm$ 0.0083
    \\
    & corrected loss &
    0.7977 $\pm$ 0.0341 &
  0.7553 $\pm$ 0.0350 &
   0.7526  $\pm$  0.0608 &
   0.6693  $\pm$  0.0721 &
   0.7000  $\pm$  0.0559 &
   0.6489  $\pm$  0.1464
    \\
    \hline
    \multirow{3}{*}{CNAE}
    & InvCal &
      0.9491 $\pm$ 0.0113 &
 0.9250 $\pm$ 0.0286 &
 0.8907 $\pm$ 0.0184 &
 0.8889 $\pm$ 0.0146 &
 0.8839 $\pm$ 0.0200 &
 0.8859 $\pm$ 0.0218
    \\
    & alter-pSVM  &
     0.9815 $\pm$ 0.0033 &
 0.9250 $\pm$ 0.0511 &
 0.9204 $\pm$ 0.0524 &
 0.8713 $\pm$ 0.0584 &
 0.8323 $\pm$ 0.0294 &
 0.8218 $\pm$ 0.0312
    \\
    & corrected loss &
     0.9463  $\pm$  0.0086 &
  0.8815  $\pm$  0.0249 &
  0.8648  $\pm$  0.0359 &
  0.6491  $\pm$  0.1689 &
  0.6847  $\pm$  0.1211 &
  0.4724  $\pm$  0.2511
    \\
    \hline
    \multirow{3}{*}{Contraceptive}
    & InvCal &
     0.6774 $\pm$ 0.0080 &
 0.6862 $\pm$ 0.0174 &
 0.6308 $\pm$ 0.0354 &
 0.6193 $\pm$ 0.0471 &
 0.6299 $\pm$ 0.0208 &
 0.5813 $\pm$ 0.0199
    \\
    & alter-pSVM &
     0.7320 $\pm$ 0.0091 &
 0.6850 $\pm$ 0.0292 &
 0.6250 $\pm$ 0.0170 &
 0.6130 $\pm$ 0.0362 &
 0.6190 $\pm$ 0.0293 &
 0.5629 $\pm$ 0.0630
    \\
    & corrected loss &
       0.6646  $\pm$ 0.0434 &
  0.7030 $\pm$ 0.0140 &
  0.5660  $\pm$ 0.0855 &
  0.5022  $\pm$ 0.0430 &
  0.5632  $\pm$ 0.0322 &
  0.5184  $\pm$ 0.0277
    \\
    \hline
    \multirow{3}{*}{German bank}
    & InvCal &
     0.7290 $\pm$ 0.0204 &
 0.7150 $\pm$ 0.0252 &
 0.6980 $\pm$ 0.0164 &
 0.7120 $\pm$ 0.0144 &
 0.7052 $\pm$ 0.0058 &
 0.7026 $\pm$ 0.0000
    \\
    & alter-pSVM &
     0.7320 $\pm$ 0.0091 &
 0.6850 $\pm$ 0.0292 &
 0.6250 $\pm$ 0.0170 &
 0.6130 $\pm$ 0.0362 &
 0.6190 $\pm$ 0.0293 &
 0.5629 $\pm$ 0.0630
    \\
    & corrected loss &
     0.6860 $\pm$ 0.0428 &
 0.6590 $\pm$ 0.0482 &
  0.6460  $\pm$ 0.0444 &
  0.5280  $\pm$  0.1201 &
  0.5233  $\pm$  0.0776 &
  0.5560  $\pm$  0.0888
    \\
  \hline
    \multirow{3}{*}{Sports article}
    & InvCal &
     0.8330 $\pm$ 0.0130 &
 0.8390 $\pm$ 0.0439 &
 0.8230 $\pm$ 0.0175 &
 0.8160 $\pm$ 0.0164 &
 0.7802 $\pm$ 0.0158 &
 0.6974 $\pm$ 0.0777
    \\
    & alter-pSVM &
     0.8270 $\pm$ 0.0076 &
 0.8420 $\pm$ 0.0396 &
 0.8210 $\pm$ 0.0248 &
 0.8160 $\pm$ 0.0175 &
 0.7828 $\pm$ 0.0370 &
 0.7259 $\pm$ 0.2034
    \\
    & corrected loss &
      0.8220  $\pm$  0.0172 &
  0.8450  $\pm$  0.0251 &
  0.8170  $\pm$  0.0117 &
  0.6940  $\pm$  0.1756 &
  0.7724  $\pm$  0.0346 &
  0.6233  $\pm$ 0.1058
    \\
 \hline
    \multirow{3}{*}{Wireless Indoor}
    & InvCal &
     0.9960 $\pm$ 0.0022 &
 0.9935 $\pm$ 0.0022 &
 0.9945 $\pm$ 0.0037 &
 0.9930 $\pm$ 0.0027 &
 0.9850 $\pm$ 0.0071 &
 0.8431 $\pm$ 0.0633
    \\
    & alter-pSVM &
     0.9970 $\pm$ 0.0027 &
 0.9985 $\pm$ 0.0014 &
 0.9970 $\pm$ 0.0033 &
 0.9980 $\pm$ 0.0011 &
 0.9100 $\pm$ 0.1929 &
 0.9948 $\pm$ 0.0039
    \\
    & corrected loss &
     0.9700 $\pm$  0.0096 &
  0.9710  $\pm$  0.0208 &
  0.9445  $\pm$  0.0165 &
  0.9060  $\pm$ 0.0280 &
 0.8375 $\pm$  0.0982 &
  0.7336  $\pm$  0.3003
    \\
  \end{tabular}
}
\end{table*}

\section{PROOFS}

We introduce the following notation. The Lipschitz constant of a loss $\ell$ is the smallest $L$ for which $\ell$ is $L$-Lipschitz, and is denoted $|\ell|$. Additionally, for any loss $\ell$ define
$$
|\ell|_0 := \max_{y \in \{-1,1\}} \ell(0,y).
$$
Note that for a margin loss $\ell(t,y) = \varphi(yt)$, we have $|\ell|_0 = \varphi(0)$.

Our proofs leverage techniques drawn from Rademacher complexity theory. This theory applies to classes of bounded function, and the following lemma ensures we are in this setting.
\begin{lemma}
\label{lem:bndloss}
Suppose $k$ is an SPD kernel on $\sX$, bounded by $K$. Then for any $R > 0$, $f\in B_k(R)$, and $x \in \sX$,
$$
| f(x) | \le RK.
$$
If in addition $\ell$ is a Lipschitz loss, then for any $R > 0$ and $f\in B_k(R)$, and any $x \in \sX$ and $y \in \{-1,1\}$,
$$
\big|\ell(f(x),y)\big| \leq |\ell|_0 + |\ell| R K.
$$
\end{lemma}
\begin{proof}
Let $\sH_k$ denote the RKHS associated to $k$. The first part follows from the reproducing property and Cauchy-Schwarz:
$$
|f(x)| = |\langle f,k(\cdot,x)\rangle| \le \|f\|_{{\cal H}_k} \|k(\cdot,x)\|_{{\cal H}_k} = RK.
$$
For the second part, by the Lipschitz continuity of $\ell$,
\begin{align*}
  \big|\ell(f(x),y)\big| & \leq \ell(0,y)+\big|\ell(f(x),y)-\ell(0,y)\big| \\
   & \leq |\ell|_0 + |\ell| |f(x)-0| \\
   & \leq |\ell|_0+|\ell| R K,
\end{align*}
where the last step uses the first part.
\end{proof}

The following lemma is used to deduce Theorems \ref{thm:gen}, \ref{thm:gen2} and \ref{thm:gen4} from Theorem \ref{thm:gen3} below. Recall that $(\ell_\a)^\rho$ denotes $\ell_\a^\rho$.

\begin{lemma}
\label{lem:lossconstants}
Let $\ell$ be a Lipschitz margin loss with $\ell(t,y) = \varphi(yt)$. Let $\rho = (\rho^+,\rho^-) \in [0,1)^2$ such that $\rho^- + \rho^+ < 1$, and $\a = (\a^+,\a^-) \in \R^2_+ $. Then
\begin{enumerate}
\item $|\ell^\rho|_0 = \varphi(0)$
\item $|\ell^{\rho}| \le |\ell| \frac{1+|\rho^+ - \rho^-|}{1-\rho^- - \rho^+}$
\item $|\ell_\a^\rho|_0 \le \frac{2\varphi(0)}{1-\rho^- - \rho^+} \max(\a^+,\a^-)$
\item $|\ell_\a^{\rho}| \le |\ell| \frac{1+|\rho^+ - \rho^-|}{1-\rho^- - \rho^+} \max(\a^+,\a^-)$.
\end{enumerate}
Now suppose that for some $\g^+, \g^- \in [0,1]$ with $\g^+ > \g^-$, we have $\a = (\frac1{\g^+ + \g^-},\frac1{2-\g^+ - \g^-})$, and $\rho = (\frac{\g^-}{\g^- + \g^+}, \frac{1 - \g^+}{2 - \g^- - \g^+})$. Then
\begin{enumerate}
\setcounter{enumi}{4}
\item $|\ell_\a^\rho|_0 \le \frac{\varphi(0)}{\g^+ - \g^-}$
\item
$|\ell_\a^{\rho}| \le  \frac{|\ell|}{\g^+ - \g^-}$.
\end{enumerate}

\end{lemma}
\begin{proof}
1. follows from definitions and 2. is a special case of 4 with $\a = (1,1)$.
To establish 3.,
\begin{multline*}
\max_y \ell_\a^\rho(0,y) = \max \left\{ \ell_\a^\rho(0,1), \ell_\a^\rho(0,-1) \right\} \\
\begin{aligned}
& = \frac{\max \left\{ (1 - \rho^-) \ell_\a(0,1) - \rho^+ \ell_\a(0,-1), (1 - \rho^+) \ell_\a(0,-1) - \rho^- \ell_\a(0,1) \right\}}{1 - \rho^- - \rho^+} \\
& = \frac{\max \left\{ (1 - \rho^-) \a^+ \ell(0,1) - \rho^+ \a^- \ell(0,-1), (1 - \rho^+) \a^- \ell(0,-1) - \rho^- \a^+ \ell(0,1) \right\}}{1 - \rho^- - \rho^+} \\
& =\frac{\varphi(0)}{1 - \rho^- - \rho^+} \max \left\{ (1 - \rho^-) \a^+ - \rho^+ \a^- , (1 - \rho^+) \a^- - \rho^- \a^+   \right\} \\
& \le \frac{\varphi(0)}{1 - \rho^- - \rho^+} \max \left\{ (1 - \rho^-) \a^+ + \rho^+ \a^- , (1 - \rho^+) \a^- + \rho^- \a^+   \right\} \\
& \le \frac{\varphi(0)}{1 - \rho^- - \rho^+} \left( (1 - \rho^-) \a^+ + \rho^+ \a^- +  (1 - \rho^+) \a^- + \rho^- \a^+   \right) \\
& = \frac{\varphi(0)}{1 - \rho^- - \rho^+} (\a^+ + \a^-) \\
& \le \frac{2\varphi(0)}{1 - \rho^- - \rho^+} \max(\a^+, \a^-).
\end{aligned}
\end{multline*}
To establish 4., let $t,t' \in R$. Then
\begin{multline*}
|\ell_\a^{\rho}(t,1) - \ell_\a^{\rho}(t',1)|  \\
\begin{aligned}
& = \frac{1}{1 - \rho^- - \rho^+} \abs{(1-\rho^-) \a^+ \ell(t,1) - \rho^+ \a^- \ell(t,-1) - [(1-\rho^-) \a^+ \ell(t',1) - \rho^+ \a^- \ell(t',-1)]} \\
& = \frac{1}{1 - \rho^- - \rho^+} \abs{(1-\rho^-) \a^+ (\ell(t,1) - \ell(t',1))  - \rho^+ \a^- (\ell(t,-1) - \ell(t',-1))} \\
&\le \frac{|\ell|}{1 - \rho^- - \rho^+} \Big[ (1-\rho^-) \a^+ + \rho^+ \a^- \Big]|t - t'| \\
&\le \frac{|\ell|}{1 - \rho^- - \rho^+} \max(\a^+, \a^-) \Big[ 1-\rho^- + \rho^+ \Big]|t - t'|
\end{aligned}
\end{multline*}
where the first inequality results from the triangle inequality and the Lipschitz property of $\ell$. Similarly,
\begin{align*}
|\ell_\a^{\rho}(t,-1) - \ell_\a^{\rho}(t',-1)|
&\le \frac{|\ell|}{1 - \rho^- - \rho^+} \Big[ (1-\rho^+) \a^- + \rho^- \a^+ \Big] |t - t'| \\
&\le \frac{|\ell|}{1 - \rho^- - \rho^+} \max(\a^+, \a^-) \Big[ 1-\rho^+ + \rho^- \Big] |t - t'|.
\end{align*}
The result now follows from
$$
\max \left\{ 1-\rho^- + \rho^+, 1-\rho^+ + \rho^-  \right\} = 1 + |\rho^+ - \rho^-|.
$$
To establish 5.,
\begin{align*}
\max_y \ell_\a^\rho(0,y) &= \max \left\{ \ell_\a^\rho(0,1), \ell_\a^\rho(0,-1) \right\} \\
& = \frac{\varphi(0)}{1 - \rho^- - \rho^+}
\max \Bigg\{ \frac{1 - \g^-}{2 - \g^- - \g^+} \frac1{\g^+ + \g^-} - \frac{\g^-}{\g^- + \g^+} \frac1{2-\g^+ - \g^-}, \\
& \qquad \qquad \qquad \qquad \qquad \qquad \frac{\g^+}{\g^- + \g^+} \frac1{2-\g^+ - \g^-} - \frac{1 - \g^+}{2 - \g^- - \g^+} \frac1{\g^+ + \g^-}  \Bigg\} \\
&\le \frac{\varphi(0)}{1 - \rho^- - \rho^+}
\max \Bigg\{ \frac{1 - \g^-}{2 - \g^- - \g^+} \frac1{\g^+ + \g^-} + \frac{\g^-}{\g^- + \g^+} \frac1{2-\g^+ - \g^-},  \\
& \qquad \qquad \qquad \qquad \qquad \qquad \frac{\g^+}{\g^- + \g^+} \frac1{2-\g^+ - \g^-} + \frac{1 - \g^+}{2 - \g^- - \g^+} \frac1{\g^+ + \g^-}  \Bigg\} \\
&=\frac{\varphi(0)}{1 - \rho^- - \rho^+} \frac1{(\g^- + \g^+)(2-\g^+ - \g^-)} \\
&= \frac{\varphi(0)}{\g^+ - \g^-},
\end{align*}
where in the last step we used Eqn. \eqref{eqn:rhogam}.

To establish 6., let $t,t' \in \R$. Arguing as in case 4., we have
$$
|\ell_\a^{\rho}(t,1) - \ell_\a^{\rho}(t',1)| \le \frac{|\ell|}{1 - \rho^- - \rho^+} \Big[ (1-\rho^-) \a^+ + \rho^+ \a^- \Big]|t - t'|.
$$
Now observe
\begin{align*}
\frac{|\ell|}{1 - \rho^- - \rho^+} \Big[ (1-\rho^-) \a^+ + \rho^+ \a^- \Big]
&= \frac{|\ell|}{1 - \rho^- - \rho^+} \Bigg[ \frac{1 - \g^-}{2 - \g^- - \g^+} \frac1{\g^+ + \g^-} + \frac{\g^-}{\g^- + \g^+} \frac1{2-\g^+ - \g^-} \Bigg] \\
&= \frac{|\ell|}{1 - \rho^- - \rho^+} \frac1{(\g^- + \g^+)(2-\g^+ - \g^-)} \\
&= \frac{|\ell|}{\g^+ - \g^-},
\end{align*}
where the last step uses Eqn. \eqref{eqn:rhogam}. The case $y=-1$ follows similarly.
\end{proof}

We may now state the master theorem, of which Theorems \ref{thm:gen}, \ref{thm:gen2} and \ref{thm:gen4} are special cases.

\begin{thm}
\label{thm:gen3}
Let $\ell$ be a Lipschitz margin loss. For $i=1,\ldots, N$, let (a) $\pi_i \in (0,1)$, (b) $P_i$ be a probability distribution on $\sX \times \{-1,1\}$, (c) $\rho_i = (\rho_i^-,\rho_i^+) \in [0,1)^2$ such that $\rho_i^- + \rho_i^+ < 1$, (d) $\a_i \in \R_+^2$, and (e) $S_i$ consist of $n_i$ iid samples of $P_i^{\rho_i}$. For all $R > 0$, $w \in \Delta^N$, and $0 < \delta \le 1$, we have that with probability at least $1-\delta$ with respect to the draws of the corrupted samples $S_1, \ldots, S_N$,
\begin{multline*}
\sup_{f \in B_{k}(R)} \abs{
\sum_{i=1}^N w_i \left( \left[ \frac1{n_i} \sum_{j=1}^{n_i} \ell_{\a_i}^{\rho_i}(f(X_{ij}),\Yt_{ij}) \right] -  \sE_{P_i,\a_i}(f) \right)} \\
\leq 2 K R \sqrt{\sum_{i=1}^N \frac{w_i^2}{n_i} |\ell_{\a_i}^{\rho_i}|^2} +
 \sqrt{\sum_{i=1}^N \frac{w_i^2}{n_i} (|\ell_{\a_i}^{\rho_i}|_0 + |\ell_{\a_i}^{\rho_i}|KR)^2  \frac{\log (2/\delta)}{2}}.
\end{multline*}
\end{thm}
\begin{proof}
We first review the following properties of the supremum which are easily verified.
\begin{enumerate}
\item[P1]
For any real-valued functions $f_1,f_2: \mathcal{X} \to \mathbb{R}$,
$$
\sup_{x}f_1(x) - \sup_{x}f_2(x)
\le \sup_{x} (f_1(x)-f_2(x)).
$$
\item[P2]
For any real-valued functions $f_1,f_2: \mathcal{X} \to \mathbb{R}$,
$$
\sup_{x} (f_1(x)+f_2(x)) \le \sup_{x}f_1(x) + \sup_{x}f_2(x).
$$
\item[P3]
$\sup(\cdot)$ is a convex function, i.e., if $(x_\lambda)_{\lambda \in \Lambda}$ and
$(x'_\lambda)_{\lambda \in \Lambda}$ are two sequences (where $\Lambda$ is possibly
uncountable), then $\forall \alpha
\in [0,1]$,
\begin{align*}
\sup_{\lambda \in \Lambda} (\alpha x_\lambda + (1-\alpha) x'_\lambda) \le \alpha \sup_{\lambda \in \Lambda} x_\lambda + (1-\alpha) \sup_{\lambda \in \Lambda}x'_\lambda.
\end{align*}
\end{enumerate}

Introduce the variable $S$ to denote all realizations $(X_{ij},\Yt_{ij})$, $1 \le i \le N, 1 \le j \le n_i$.
Since the realizations $X_{ij}$ are independent, we can apply the Azuma-McDiarmid
inequality \citep{mcdiarmid} to the function
$$
\xi(S) := \sup_{f \in B_{k}(R)}
\sum_{i=1}^N w_i \left( \left[ \frac1{n_i} \sum_{j=1}^{n_i} \ell_{\a_i}^{\rho_i}(f(X_{ij}),\Yt_{ij})\right] - \sE_{P_i,\a_i}(f) \right) .
$$
In particular, by Lemmas \ref{lem:bndloss} and \ref{lem:lossconstants}, every loss $\ell_{\a_i}^{\rho_i}$ is bounded by
$K_i := |\ell_{\a_i}^{\rho_i}|_0 + |\ell_{\a_i}^{\rho_i}|KR$. If we modify $S$ by replacing some $(X_{ij},\Yt_{ij})$ with another $(X_{ij}',\Yt_{ij}')$, while leaving all other values in $S$ fixed, then (by P1)
$\xi$ changes by at most $\frac{w_i K_i}{n_i}$, and we obtain
that with probability $1-\delta$ over the draw of
$S_1, \ldots, S_N$,
$$
\xi-\e{\xi} \leq \sqrt{\sum_{i=1}^N \frac{w_i^2 (|\ell_{\a_i}^{\rho_i}|_0 + |\ell_{\a_i}^{\rho_i}|KR)^2}{n_i} \frac{\log (1/\delta)}{2}}.
$$

To bound $\e{\xi}$ we will use ideas from Rademacher complexity theory. Thus let $S'$ denote a separate (ghost) sample of corrupted data $(X_{ij}',\Yt_{ij}') \stackrel{iid}{\sim} P_i^{\rho_i}$, $i=1,\ldots, N$, $j=1,\ldots, n_i$, independent of the realizations in $S$. Let $\widehat{\mbe}_S[f]$ be shorthand for $\sum_i \frac{w_i}{n_i} \sum_j \ell_{\a_i}^{\rho_i}(f(X_{ij}),\Yt_{ij}).$ Denote by $(\eps_{ij})$ $1\leq i \leq N, 1 \le j \le n_i$, iid Rademacher variables (independent from everything else), and let $\mbe_{(\eps_{ij})}$ denote the expectation with respect to these variables. We have
\begin{align*}
\e{\xi} & = \ee{S}{\sup_{f \in B_k(R)} \sum_{i=1}^N w_i \left( \left[ \frac1{n_i} \sum_{j=1}^{n_i} \ell_{\a_i}^{\rho_i}(f(X_{ij}),\Yt_{ij})\right] - \sE_{P_i,\a_i}(f) \right)} \\
&= \ee{S}{\sup_{f \in B_k(R)} \Bigg(  \widehat{\mbe}_S[f]
 - \ee{S'}{\widehat{\mbe}_{S'}[f]} \Bigg)} \\
& \qquad \text{(by Lemmas \ref{lem:corrected} and \ref{lem:bndloss})} \\
&\le \ee{S,S'}{\sup_{f \in B_k(R)} \Bigg( \widehat{\mbe}_S[f]
 - \widehat{\mbe}_{S'}[f]  \Bigg)} \\
& \qquad \text{(by P3 and Jensen's inequality)} \\
&= \ee{S,S'}{\sup_{f \in B_k(R)} \Bigg( \sum_{i=1}^N \frac{w_i}{n_i} \sum_{j=1}^{n_i}\ell_{\a_i}^{\rho_i}(f(X_{ij}),\Yt_{ij}) - \ell_{\a_i}^{\rho_i}(f(X_{ij}'),\Yt_{ij}') \Bigg)} \\
&= \ee{S,S',(\eps_{ij})}{\sup_{f \in B_k(R)} \Bigg( \sum_{i=1}^N \frac{w_i}{n_i} \sum_{j=1}^{n_i} \eps_{ij} \Big(\ell_{\a_i}^{\rho_i}(f(X_{ij}),\Yt_{ij}) - \ell_{\a_i}^{\rho_i}(f(X_{ij}'),\Yt_{ij}')\Big) \Bigg)} \\
& \qquad \text{(for all $i, j$, $(X_{ij},\Yt_{ij})$ and $(X_{ij}',\Yt_{ij}')$ are iid, and $\eps_{ij}$ are symmetric)}\\
&\le \ee{S,S',(\eps_{ij})}{\sup_{f \in B_k(R)} \sum_{i=1}^N \frac{w_i}{n_i} \sum_{j=1}^{n_i} \eps_{ij} \ell_{\a_i}^{\rho_i}(f(X_{ij}),\Yt_{ij})} \\
&\qquad \qquad + \ee{S,S',(\eps_{ij})}{\sup_{f \in B_k(R)} \sum_{i=1}^N \frac{w_i}{n_i} \sum_{j=1}^{n_i} (-\eps_{ij}) \ell_{\a_i}^{\rho_i}(f(X_{ij}'),\Yt_{ij}')}\\
& \qquad \text{(by P2)} \\
&= 2 \mbe_{S}\ee{(\eps_{ij})}{\sup_{f \in B_k(R)} \sum_{i=1}^N \frac{w_i}{n_i} \sum_{j=1}^{n_i} \eps_{ij} \ell_{\a_i}^{\rho_i}(f(X_{ij}),\Yt_{ij})}.
\end{align*}

To bound the innermost expectation we use the following result from \citet{meir03jmlr}.
\begin{lemma} Suppose $\set{\phi_t}, \set{\psi_t}, t=1,\ldots,T$,
are two sets of functions on a set $\Theta$ such that for each $t$ and
$\theta,\theta' \in \Theta, |\phi_t(\theta)-\phi_t(\theta')| \le
|\psi_t(\theta)-\psi_t(\theta')|$.  Then for all functions $c:\Theta
\rightarrow \R,$
\begin{equation*}
\mbe_{(\eps_t)} \brac{\sup_\theta \set{c(\theta)+\sum_{t=1}^T \eps_t
\phi_t(\theta)}}
\le \mbe_{(\eps_t)} \brac{\sup_\theta \set{c(\theta)+\sum_{t=1}^T \eps_t
\psi_t(\theta)}}.
\end{equation*}
\label{lemmaZhang}
\end{lemma}

Switching from the single index $t$ to our two indices $i$ and $j$, we apply the lemma with $\Theta = \sF$, $\theta = f$, $c(\theta) = 0$, $\phi_{ij}(\theta) = \frac{w_i}{n_i} \ell_{\a_i}^{\rho_i}(f(X_{ij}),\Yt_{ij})$, and $\psi_{ij}(\theta) = \frac{w_i}{n_i} |\ell_{\a_i}^{\rho_i}| f(X_{ij})$ to obtain
\begin{equation*}
\e{\xi} \le 2\mbe_{S} \ee{(\eps_{ij})}{\sup_{f \in B_{k}(R)}  \sum_{i=1}^N \frac{w_i |\ell_{\a_i}^{\rho_i}|}{n_i} \sum_{j=1}^{n_i} \eps_{ij} f(X_{ij})}.
\end{equation*}
This inner expectation may be bounded by a modification of the standard bound of Rademacher complexity for kernel classes. For conciseness introduce $a_i := \frac{w_i |\ell_{\a_i}^{\rho_i}|}{n_i}$. Then
\begin{align}
\mbe_{(\eps_{ij})} \Bigg[ \sup_{f \in B_{k}(R)} & \sum_{i,j} a_i \eps_{ij} f(X_{ij}) \Bigg] \nonumber \\
&= \ee{(\eps_{ij})}{\sup_{f \in B_{k}(R)} \sum_{i,j} a_i \eps_{ij} \langle f, k(\cdot, X_{ij}) \label{eqn:repr} \rangle} \\
&= \ee{(\eps_{ij})}{\sup_{f \in B_{k}(R)} \left\langle f, \sum_{i,j} a_i \eps_{ij} k(\cdot, X_{ij}) \right\rangle} \nonumber \\
&= \ee{(\eps_{ij})}{\left\langle R \frac{\sum_{i,j} a_i \eps_{ij} k(\cdot, X_{ij})}{\| \sum_{i,j} a_i \eps_{ij} k(\cdot, X_{ij}) \|}, \sum_{i,j} a_i \eps_{ij} k(\cdot, X_{ij}) \right\rangle} \label{eqn:cs} \\
&= R \ee{(\eps_{ij})}{\sqrt{\Bigg\| \sum_{i,j} a_i \eps_{ij} k(\cdot, X_{ij}) \Bigg\|^2}} \nonumber \\
&\le R \sqrt{\ee{(\eps_{ij})}{\Bigg\| \sum_{i,j} a_i \eps_{ij} k(\cdot, X_{ij}) \Bigg\|^2}} \label{eqn:jensen} \\
&= R \sqrt{\sum_{i,j} a_i^2 \| k(\cdot,X_{ij}) \|^2} \label{eqn:rad} \\
&= RK \sqrt{\sum_{i=1}^N \frac{w_i^2  |\ell_{\a_i}^{\rho_i}|^2}{n_i}} \label{eqn:repr2},
\end{align}
where \eqref{eqn:repr} uses the reproducing property, \eqref{eqn:cs} is the condition for equality in Cauchy-Schwarz, \eqref{eqn:jensen} is Jensen's inequality, \eqref{eqn:rad} follows from independence of the Rademacher random variables, and \eqref{eqn:repr2} follows from the reproducing property and the bound on the kernel.

All of the above establishes a high probability bound on $\xi$. To get a bound with an absolute value (as in the theorem statement), we may argue analogously to bound the reverse direction, also with probability at least $1-\delta$. Substituting $\delta \to \delta/2$ combines the two directions, giving a bound on the supremum of the absolute value of the generalization error, thus completing the proof.
\end{proof}

\subsection{Proof of Theorem \ref{thm:gen}}

To prove Theorem \ref{thm:gen} we apply Theorem \ref{thm:gen3} with $P_i = P$ and $\a_i = (1,1)$. The result now follows from Lemma \ref{lem:lossconstants} parts 1 and 2, the fact that for $R > \varphi(0)/(|\ell|K)$ it holds $\varphi(0) \le |\ell|KR \le |\ell^{\rho_i}|KR$, and the fact that $1 < \log(2/\delta)/2$ when $\delta \le \frac14$.

\subsection{Proof of Theorem \ref{thm:gen2}}

To prove Theorem \ref{thm:gen2} we apply Theorem \ref{thm:gen3} and note that for the given $\a_i$, $\sE_{P_i,\a_i} = \sE_{\pm}$. Applying Lemma \ref{lem:lossconstants} parts 3 and 4 we have that
$$
|\ell_{\a_i}^{\rho_i}|_0 + |\ell_{\a_i}^{\rho_i}| KR \le \max(\a_i^+,\a_i^-)\frac{1 + |\rho_i^+ - \rho_i^-|}{1 - \rho_i^+ - \rho_i^-} (2\varphi(0) + |\ell|KR).
$$
The result now follows from the observation that $\max(\a_i^+,\a_i^-) = \frac12 \frac1{\min(\pi_i,1-\pi_i)}$, the fact that for $R > 2\varphi(0)/(|\ell|K)$ it holds $2\varphi(0) \le |\ell|KR$, and the fact that $1 < \log(2/\delta)/2$ when $\delta \le \frac14$.

\subsection{Proof of Theorem \ref{thm:gen4}}

To prove Theorem \ref{thm:gen4} we apply Theorem \ref{thm:gen3} and note that for the given $\a_i$, $\sE_{P_i,\a_i} = \sE_{\pm}$. The result now follows from Lemma \ref{lem:lossconstants} parts 5 and 6, the fact that for $R > \varphi(0)/(|\ell|K)$ it holds $\varphi(0) \le |\ell|KR$, and the fact that $1 < \log(2/\delta)/2$ when $\delta \le \frac14$.

\subsection{Proof of Theorem \ref{thm:llpconsist}}

Under the bounded noise condition, it holds for all three settings that with probability at least $1-\delta$,
$$
\sup_{f \in B_k(R)} \left| \sE(f) - \widehat{\sE}(f) \right| \le \frac{8KRL}{c_0}\sqrt{\left(\sum_i n_i\right)^{-1} \frac{\log(2/\delta)}{2}},
$$
where $c_0$ is the noise level lower bound.
This can be seen by plugging in the optimal weights and simplifying.

Denote
\begin{align*}
J(f) &= \sum_{i=1}^{N} \frac{w_i}{n_i} \sum_{j=1}^{n_i} \ell_{\a_i}^{\rho_i}(f(X_{ij}),\Yt_{ij}) + \lambda
\|f\|^2\\
&=\widehat{\sE}(f) + \lambda\|f\|^2.
\end{align*}
Observe that $J(\widehat{f}) \leq J(0) = \varphi(0)$. Therefore $\lambda
\|\widehat{f}\|^2 \leq \varphi(0) - \widehat{\sE}(\widehat{f})\leq \varphi(0)$ and
so $\|\widehat{f}\|^2 \leq \varphi(0)/\lambda$.

Set $R = \sqrt{\varphi(0)/\lambda}$. Note that $R$ grows asymptotically because $\lambda$ shrinks. We just saw that $\widehat{f} \in B_k(R)$.

Let $\epsilon>0$. Fix $f_\epsilon \in \sH_k$ s.t. $\sE(f_\epsilon) \leq \sE^* +
\epsilon/2$, possible since $k$ is universal \citep{steinwart08}. Note that $f_\epsilon \in B_k(R)$ for $\sum_i n_i$ sufficiently
large. In this case the generalization error bound implies that with probability $\geq
1-\delta$,
\begin{align*}
\sE(\widehat{f}) &\leq \widehat{\sE}(\widehat{f}) +
\frac{8KRL}{c_0}\sqrt{\left(\sum_i n_i\right)^{-1} \frac{\log(2/\delta)}{2}}\\\nonumber
&\leq \widehat{\sE}(f_\epsilon) + \lambda\|f_\epsilon\|^2 -
\lambda\|\widehat{f}\|^2 + \frac{8KRL}{c_0}\sqrt{\left(\sum_i n_i\right)^{-1} \frac{\log(2/\delta)}{2}}\\\nonumber
&\leq \widehat{\sE}(f_\epsilon) + \lambda\|f_\epsilon\|^2 +
\frac{8KRL}{c_0}\sqrt{\left(\sum_i n_i\right)^{-1} \frac{\log(2/\delta)}{2}} \\
&\leq \sE(f_\epsilon) + \lambda\|f_\epsilon\|^2 +
\frac{16KRL}{c_0}\sqrt{\left(\sum_i n_i\right)^{-1} \frac{\log(2/\delta)}{2}}.
\end{align*}
Taking $\delta = (n_1 + \cdots + n_N)^{-1}$, the result now follows.

\subsection{Proof of Proposition \ref{prop:conv2}}

Through simple algebra it can be shown that for each $i$,
$$
(\ell_{\a_i})^{\rho_i}(t,1) = \frac1{\g_i^+ - \g_i^-} [(1 - \g_i^-)\ell(t,1) - \g_i^- \ell(t,-1)]
$$
and
$$
(\ell_{\a_i})^{\rho_i}(t,-1) = \frac1{\g_i^+ - \g_i^-} [\g_i^+\ell(t,-1) - (1 - \g_i^+) \ell(t,1)].
$$
Applying \eqref{eqn:second2} it now follows that
\begin{align*}
J''(t) &= \ell''(t,1) \sum_i \frac{w_i}{n_i(\g_i^+ - \g_i^-)} \sum_j \Bigg[ (1 - 2\g_i^-) \ind{\Yt_{ij}=1} + (2\g_i^+ - 1) \ind{\Yt_{ij}=-1} \Bigg] \\
&= \ell''(t,1) \sum_i \frac{w_i}{n_i(\g_i^+ - \g_i^-)} \Bigg[ (1 - 2\g_i^-) n_i^+ + (2\g_i^+ - 1) n_i^- \Bigg].
\end{align*}
The result now follows.


\end{document}